\documentclass[11pt]{article}
 \usepackage{amsmath, amsthm, amssymb, bbm, setspace,bigints}
 \usepackage[margin=1 in]{geometry}
\usepackage{caption}
\usepackage{subcaption}
\usepackage[toc,page]{appendix}
\usepackage{cite}         
\usepackage[pdftex]{graphicx}
\usepackage{pdfpages}
\usepackage{epstopdf}
\usepackage{booktabs}
\usepackage{amsmath,tikz,url}
\usetikzlibrary{arrows.meta,positioning}
\usepackage{authblk}

\usepackage{amsthm,bbm,amssymb}
\usepackage{amsmath}
\usepackage{natbib}
\usepackage[colorlinks,citecolor=blue,urlcolor=blue,filecolor=blue,backref=page]{hyperref}
\usepackage{graphicx}

\usepackage[utf8]{inputenc}
\usepackage{amsmath,amssymb,natbib,xcolor,multicol,url,mhequ,amsthm}
\usepackage{graphicx,bbm,xr}
\usepackage{booktabs,epstopdf,color}
\usepackage[space]{grffile}
\usepackage{lineno}
\usepackage[plain,noend]{algorithm2e}
\usepackage{multirow}

\def\T{\mathrm{\scriptscriptstyle{T}}}
\newtheorem{proposition}{Proposition}

\def\mb{\mathbb}

\def\T{{\mathrm{\scriptscriptstyle T} }}

\def\H{{\mathrm{\scriptscriptstyle H} }}
\newcommand{\mi}{\mathrm{i}}

\newcommand{\be}{\begin{equs}}
\newcommand{\ee}{\end{equs}}

\numberwithin{equation}{section}
\theoremstyle{plain}
\newtheorem{theorem}{Theorem}
\newtheorem{assumption}{Assumption}
\newtheorem{remark}{Remark}

\newtheorem{lemma}{Lemma}

\title{Stability of Sequential and Parallel Coordinate Ascent Variational Inference}

\author[1]{Debdeep Pati\thanks{dpati2@wisc.edu}}
\affil[1]{Department of Statistics,  University of Wisconsin-Madison,  Madison, WI 53706}

\begin{document}
\maketitle

\begin{abstract}
We highlight a striking difference in behavior between two widely used variants of coordinate ascent variational inference: the sequential and parallel algorithms. While such differences were known in the numerical analysis literature in simpler settings, they remain largely unexplored in the optimization-focused literature on variational inference in more complex models. Focusing on the moderately high-dimensional linear regression problem, we show that the sequential algorithm, although typically slower, enjoys convergence guarantees under more relaxed conditions than the parallel variant, which is often employed to facilitate block-wise updates and improve computational efficiency.

\end{abstract}

\noindent\textbf{Keywords:} Bayesian; Gauss-Siedel; Jacobi; parallel; regression; sequential; variational inference

\section{Introduction}

Variational inference has emerged over the past two decades as a scalable framework for approximate Bayesian computation. Despite its widespread empirical success \citep{blei2017variational} and strong statistical optimality guarantees \citep{pati2018statistical,wang2019frequentist,yang2020alpha,alquier2020concentration}, the theoretical understanding of its optimization landscape is still evolving. Recent advances have begun to shed light on this landscape in specific model settings \citep{wang2006convergence,zhang2017theoretical,mukherjee2018mean,ghorbani2018instability,plummer2020dynamics,celentano2021local,ghosh2022statistical}, for restricted variational families such as the mean-field approximation \citep{bhattacharya2025convergence}, and under structural assumptions on the posterior distribution, for instance, log-concavity \citep{arnese2024convergence}.

In this article, we revisit the coordinate ascent variational inference in mean-field inference for Bayesian linear regression \citep{10.1214/12-BA703}, motivated by its empirical success \citep{10.1214/12-BA703,zabad2023fast} and statistical optimality guarantees \citep{ray2022variational}. Since the only existing work on the convergence of coordinate ascent \citep{bhattacharya2025convergence} primarily focuses on the two-block case, little is known about the behavior of the coordinate-ascent algorithm  employed in 
\cite{10.1214/12-BA703}, which uses a $p$-block coordinate ascent where $p$ is the number of covariates. Coordinate ascent algorithms optimize the evidence lower bound by updating each variational factor sequentially while keeping the others fixed. Variants of this basic scheme have appeared widely in the literature, with two of the most prevalent being the \emph{sequential} and \emph{parallel} coordinate ascent algorithms \citep{huang2016variational}. In the sequential algorithm, when updating the density associated with coordinate \(i\) at iteration \(t\), the most recent updates for coordinates \(\{1, \ldots, i-1\}\) are used. In contrast, the parallel version updates all coordinates simultaneously using their values from iteration \((t-1)\). A schematic comparison of the two approaches is provided in Figure~\ref{pic:seq_par}. It is often argued that the parallel version naturally facilitates simultaneous updates across coordinates, which can yield improved scalability and computational efficiency.  It is worth emphasizing that, although Section 5 of \cite{bhattacharya2025convergence} addresses coordinate ascent algorithms with more than two blocks, the convergence result presented in Theorem 5.2 therein pertains exclusively to the parallel variant and does not distinguish between the sequential and parallel versions.

\begin{figure}[htbp]
\centering
\begin{minipage}{0.48\textwidth}
\centering
\begin{tikzpicture}[>=Stealth,thick,scale=0.5, every node/.style={transform shape}]

\node[circle,draw,fill=blue!70,inner sep=1.2pt,label=above:{\textcolor{blue}{$q_1^{(t)}$}}] (q1t) at (0,2) {};
\node[circle,draw,fill=blue!70,inner sep=1.2pt,label=below:{\textcolor{blue}{$q_2^{(t)}$}}] (q2t) at (1,0) {};

\draw[->] (q2t) -- (q1t);
\draw[->] (q1t) -- (q2t);

\node[circle,draw,fill=blue!70,inner sep=1.2pt,label=above:{\textcolor{blue}{$q_1^{(t+1)}$}}] (q1t1) at (3,2) {};
\node[circle,draw,fill=blue!70,inner sep=1.2pt,label=below:{\textcolor{blue}{$q_2^{(t+1)}$}}] (q2t1) at (4,0) {};

\draw[->] (q2t) -- (q1t1);
\draw[->] (q1t1) -- (q2t1);

\node[circle,draw,fill=blue!70,inner sep=1.2pt,label=above:{\textcolor{blue}{$q_1^{(t+2)}$}}] (q1t2) at (6,2) {};
\node[circle,draw,fill=blue!70,inner sep=1.2pt,label=below:{\textcolor{blue}{$q_2^{(t+2)}$}}] (q2t2) at (7,0) {};

\draw[->] (q2t1) -- (q1t2);
\draw[->] (q1t2) -- (q2t2);

\node[circle,draw,fill=blue!70,inner sep=1.2pt] (next) at (9,2) {};
\draw[->] (q2t2) -- (next);

\end{tikzpicture}
\end{minipage}%
\hfill
\begin{minipage}{0.48\textwidth}
\centering
\begin{tikzpicture}[>=Stealth,thick,scale=0.5, every node/.style={transform shape}]

\node[circle,draw,fill=blue!70,inner sep=1.2pt,label=above:{\textcolor{blue}{$q_1^{(t)}$}}] (q1t) at (0,2) {};
\node[circle,draw,fill=blue!70,inner sep=1.2pt,label=below:{\textcolor{blue}{$q_2^{(t)}$}}] (q2t) at (0,0) {};

\node[circle,draw,fill=blue!70,inner sep=1.2pt,label=above:{\textcolor{blue}{$q_1^{(t+1)}$}}] (q1t1) at (3,2) {};
\node[circle,draw,fill=blue!70,inner sep=1.2pt,label=below:{\textcolor{blue}{$q_2^{(t+1)}$}}] (q2t1) at (3,0) {};

\draw[->] (q1t) -- (q1t1);
\draw[->] (q2t) -- (q2t1);

\node[circle,draw,fill=blue!70,inner sep=1.2pt,label=above:{\textcolor{blue}{$q_1^{(t+2)}$}}] (q1t2) at (6,2) {};
\node[circle,draw,fill=blue!70,inner sep=1.2pt,label=below:{\textcolor{blue}{$q_2^{(t+2)}$}}] (q2t2) at (6,0) {};

\draw[->] (q1t1) -- (q1t2);
\draw[->] (q2t1) -- (q2t2);

\end{tikzpicture}
\end{minipage}
\caption{Comparison of sequential (left) and parallel (right) versions for $p=2$. }
\label{pic:seq_par}
\end{figure}
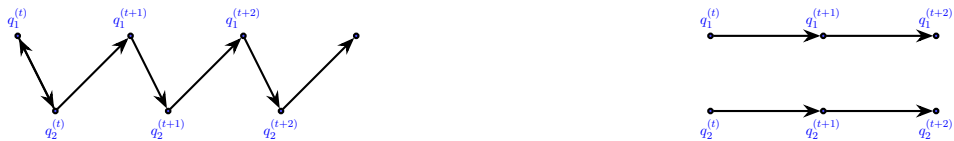

There has been limited understanding of the qualitative differences between these two coordinate ascent approaches, apart from a few illustrative examples. The first example arises when the target distribution is multivariate Gaussian. In this case, the sequential algorithm corresponds to the classical Gauss--Seidel iteration, whereas the parallel algorithm coincides with the Jacobi iteration, both of which are well-studied in numerical linear algebra \citep{golub2013matrix}. Notably, as shown in Theorems 10.1.1 and 10.1.2 and  of \cite{golub2013matrix}, the Gauss-Seidel iteration converges without additional assumptions on the target covariance, whereas the Jacobi iteration requires the covariance matrix to satisfy a diagonal dominance condition. This observation is further highlighted in Section 5 of \cite{bhattacharya2025convergence}.  The second example concerns the estimation of the normalizing constant in Ising models using mean-field variational inference  \citep{plummer2020dynamics}  where tools from dynamical systems theory are employed to study the convergence properties of the sequential and parallel coordinate ascent algorithms. Focusing on the Ising model with two nodes, their analysis reveals notable differences between the two algorithms in regions where the objective function is non-convex. In particular, the parallel algorithm exhibits periodic oscillatory behavior, which is entirely absent in the sequential variant.

Although the sequential algorithm typically converges more slowly, the two simple examples above indicate that it may offer certain stability advantages in more complex settings where $d$-block coordinate ascent variational inference (CAVI) is commonly employed. This observation motivates a deeper investigation in the more practical context of Bayesian sparse linear regression, where both sequential \citep{ray2022variational,ge2025variational} and parallel  \citep{huang2016variational,10.1214/17-EJS1332,yang2020alpha,dasgupta2023approximate} coordinate ascent algorithms have been implemented in practice. 
Our analysis reveals an interesting phenomenon when both algorithms are reformulated as fixed-point iterations. Specifically, under a Gaussian design, the spectral radius of the Jacobian matrix evaluated at the fixed point associated with the parallel update can exceed one with non-negligible probability in realistic scenarios, whereas the corresponding spectral radius for the sequential update remains strictly between zero and one under mild conditions. From a practical perspective, this implies that the parallel algorithm is prone to divergence under realistic settings, as illustrated by a running example presented in \S \ref{sec:run}.

\section{Coordinate ascent variational inference in sparse Bayesian linear regression}
Consider the linear regression model
\be
Y = X\beta + \epsilon, \quad \epsilon \sim \mbox{N}_n(0, \sigma^2),
\ee
where $\beta \in \mathbb{R}^p$ and $p \leq n$. Throughout the remainder of the paper, we assume that the noise variance $\sigma^2$ is fixed and known. This assumption is adopted solely to streamline the exposition and isolate the algorithmic and inferential properties of the variational procedures. 

We impose a spike-and-slab prior on the regression coefficients,
\[
p(\beta) = \prod_{j=1}^p p(\beta_j), 
\quad 
\beta_j \sim (1-\pi)\delta_0 + \pi N(0,\tau^{-1}),
\]
where $\pi \in (0,1)$ controls the overall sparsity level and $\tau>0$ denotes the slab precision.  We assume $\tau$ to be a constant, but allow $\pi$ to depend on $n, p$. This prior induces exact zeros with positive probability while allowing nonzero coefficients to be adaptively shrunk, making it particularly suitable for high-dimensional sparse regression. It is well known from \cite{CastilloSchmidtHieberVanDerVaart2015,yang2020alpha,ray2022variational} that, for consistent recovery and model selection of the true signals in high-dimensional settings, the prior inclusion probability must satisfy $\pi \asymp p^{-A}$ for some constant $A>0$.

Posterior inference is carried out using a mean-field variational approximation of the form $
q(\beta) = \prod_{j=1}^p q_j(\beta_j)$, 
which factorizes across coordinates and leads to closed-form coordinate ascent updates. We first derive the sequential coordinate ascent variational inference (CAVI) updates and then contrast them with their parallel counterparts.

\subsection{Sequential coordinate ascent}
Define $
a_j = \|X_j\|^2/\sigma^2 + \tau$, 
where $X_j$ denotes the $j$th column of the design matrix $X$. From Proposition~\ref{prop:cavi}, denoting $\mbox{logit}(x) := \log \{x/(1-x)\}$ for $x \in (0, 1)$, the coordinate-wise updates for the variational mean and inclusion probability are given by
\begin{align} \label{eq:seq_CS1}
\mu_j^{(t+1)} &= \frac{1}{\sigma^2a_j} \Big[ \langle X_j, y \rangle  
- \sum_{l < j} X_j^{\T} X_l \alpha_l^{(t)} \mu_l^{(t+1)} 
- \sum_{l > j} X_j^{\T} X_l \alpha_l^{(t)} \mu_l^{(t)} \Big],
\quad j = 1, \ldots, p, \\
\mbox{logit}(\alpha_j^{(t)}) & = \mbox{logit}(\pi) 
+ \frac{1}{2} \log \bigg(\frac{\tau}{a_j}\bigg) 
+ \frac{a_j (\mu_j^{(t)})^2}{2}, 
\quad \alpha_j^{(t)} := \psi_j(\mu_j^{(t)}). \label{eq:seq_CS2}
\end{align}
The update for $\mu_j^{(t+1)}$ exhibits a Gauss--Seidel structure: newly updated coordinates $\mu_l^{(t+1)}$ for $l<j$ are immediately reused, whereas coordinates $l>j$ retain their previous values. This sequential dependence is a defining feature of classical CAVI algorithms and is often associated with improved numerical stability and faster convergence relative to fully parallel updates.

To simplify notation and highlight the underlying linear-algebraic structure, define $L^* = \mbox{lower}(X^{\T} X)$ and $U^* = \mbox{upper}(X^{\T} X)$, so that $(L^*)^{\T} = U^*$. Let
\[
D_{\alpha^{(t)}} = \mbox{diag}(\alpha_1^{(t)}, \alpha_2^{(t)}, \ldots, \alpha_p^{(t)}),
\quad
L = L^*D_{\alpha^{(t)}}, 
\quad 
U = U^*D_{\alpha^{(t)}}.
\]
Further define $D = \sigma^2 \mbox{diag}(a_1, \ldots, a_p)$ and $f = X^{\T}Y$. With this notation, the sequential update \eqref{eq:seq_CS1}--\eqref{eq:seq_CS2} can be written compactly as
\begin{align}
\mu^{(t+1)} & = D^{-1} [f - L \mu^{(t+1)} - U \mu^{(t)} ] \nonumber\\
 & = -(I + D^{-1} L)^{-1} D^{-1} U \mu^{(t)} 
 + (I + D^{-1}L)^{-1} D^{-1} f \nonumber\\
& = -(D+ L^*D_{\alpha^{(t)}})^{-1}  (L^*)^T D_{\alpha^{(t)}} \mu^{(t)} 
 + (D+ L^*D_{\alpha^{(t)}})^{-1} f  \nonumber\\
& := G(\mu^{(t)}) \mu^{(t)} + H(\mu^{(t)}), 
\label{eq:seq_final}
\end{align}
where
\begin{align}
G(\mu^{(t)}) = -(D+ L^*D_{\alpha^{(t)}})^{-1}  (L^*)^T D_{\alpha^{(t)}},
\qquad
H(\mu^{(t)}) = (D+ L^*D_{\alpha^{(t)}})^{-1} f.
\end{align}

Equation~\eqref{eq:seq_final} reveals that the sequential CAVI algorithm can be interpreted as a nonlinear fixed-point iteration, where the non-linear operator $G(\mu^{(t)})$ depends implicitly on the current iterate $\mu^{(t)}$ through the variational inclusion probabilities $\alpha^{(t)}$. This dependence distinguishes the variational dynamics from classical linear Gauss--Seidel iterations and plays a central role in the convergence behavior analyzed in subsequent sections.

\subsection{Parallel coordinate ascent}
We now turn to the parallel CAVI scheme, in which all coordinates are updated simultaneously using information from the previous iteration. From Proposition~\ref{prop:cavi}, the updates take the form
\begin{align} \label{eq:par_CS1}
\mu_j^{(t+1)} &= \frac{1}{\sigma^2a_j} \Big[ \langle X_j, y \rangle  
- \sum_{l \neq j} X_j^{\T} X_l \alpha_l^{(t)} \mu_l^{(t)}  \Big],
\quad j = 1, \ldots, p, \\
\mbox{logit}(\alpha_j^{(t)}) & = \mbox{logit}(\pi) 
+ \frac{1}{2} \log \bigg(\frac{\tau}{a_j}\bigg) 
+ \frac{a_j (\mu_j^{(t)})^2}{2}, 
\quad \alpha_j^{(t)} := \psi_j(\mu_j^{(t)}). \label{eq:par_CS2}
\end{align}
Unlike the sequential scheme, the parallel updates do not exploit intermediate coordinate updates within an iteration. As a result, they correspond to a Jacobi-type iteration, which is often easier to parallelize and is well suited for distributed computing architectures.
Using the same matrix notation as before, the parallel update \eqref{eq:par_CS1}--\eqref{eq:par_CS2} can be written succinctly as
\begin{align}\label{eq:par_final}
\mu^{(t+1)} = D^{-1} [f - L \mu^{(t)} - U \mu^{(t)} ].
\end{align}
Comparing \eqref{eq:seq_final} and \eqref{eq:par_final}, the essential distinction between the two algorithms lies in the treatment of the lower-triangular component $L$. The sequential scheme effectively preconditions the update through $(D + L^*D_{\alpha^{(t)}})^{-1}$, whereas the parallel scheme applies a simpler diagonal preconditioner $D^{-1}$. This difference has important consequences on the convergence of the algorithms.

\section{A running example} \label{sec:run}
We illustrate the qualitative behavior of the two schemes in a simple synthetic setting.\footnote{\url{https://github.com/debdeepuw/Sequential_parallel_CAVI}} Consider the model in (1) with Gaussian design $X \in \mathbb{R}^{n \times p}$ having independent $N(0,1)$ entries. We fix $(n,p,s) = (200,50,25)$, where $s$ denotes the number of nonzero regression coefficients, and generate $\beta^\ast = (1,\ldots,1,0,\ldots,0) \in \mathbb{R}^p$ with $s$ active components and $\sigma^2 = 1$. 
We run the sequential CAVI algorithm defined by \eqref{eq:seq_CS1}--\eqref{eq:seq_CS2} with $\pi = 0.5, \tau = 1$, producing iterates $\mu^{(t)}$. Figure~\ref{fig:elbo_seq} displays the resulting variational mean together with the evidence lower bound (ELBO) evaluated using the exact expression in (24) in \cite{10.1214/12-BA703}.  
\begin{figure}[htbp!]
\centering
\includegraphics[width=0.8\textwidth]{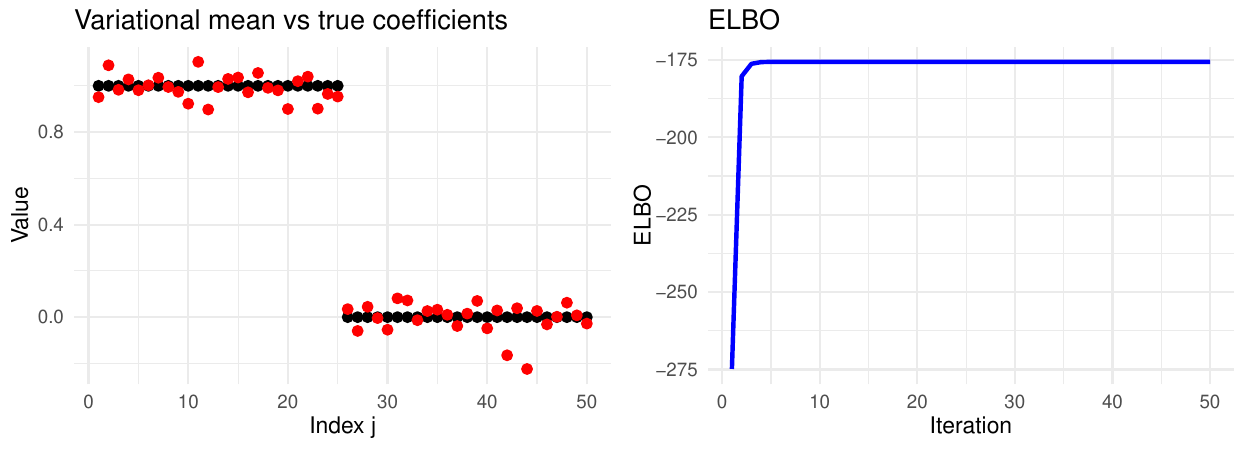}
\caption{Sequential CAVI for $(n,p,s)=(200,50,25)$. Left: variational mean $\mu_j$ (red) versus true coefficients $\beta_j$ (black). Right: ELBO as a function of iteration.}
\label{fig:elbo_seq}
\end{figure}

The variational mean $\mu^{(t)}$ quickly stabilizes and recovers the underlying sparsity pattern, with clear separation between active and inactive coordinates. More importantly, the ELBO increases monotonically and reaches a plateau within a small number of iterations.  In contrast, the parallel CAVI iteration (8) does not exhibit such stability. In this example, the corresponding ELBO sequence is highly unstable and diverges, and is therefore omitted. This discrepancy motivates a closer examination of the local dynamics of the two update schemes.

\section{Local asymptotic stability analysis}

Many iterative algorithms in statistics and optimization, such as sequential and parallel versions of the coordinate ascent variational inference as in \eqref{eq:seq_final} and \eqref{eq:par_final} can be viewed as fixed point iterations whose convergence depends on the chosen update scheme. In this section, we develop a framework for analyzing the local convergence properties of such algorithms, with particular emphasis on the distinction between sequential and parallel updates.
We consider fixed point iterations of the form
\begin{eqnarray}
x_{k+1} = g(x_k,\theta), \label{eq:maps}
\end{eqnarray}
where $x_k \in U \subset \mathbb{R}^n$ and $\theta \in V \subset \mathbb{R}^p$ is a parameter.  
Given an initial condition $x_0$, repeated application of $g$ generates an orbit
\[
\{x_0, g(x_0;\theta), g^2(x_0;\theta), \ldots\},
\]
which describes the evolution of the algorithm.  The sequential and the parallel versions in \eqref{eq:seq_final} and \eqref{eq:par_final} correspond to different choices of the map $g$, and therefore generate different orbits even when they target the same fixed point. Throughout this section we fix $\theta$ and write $g(x)\equiv g(x;\theta)$. Suppose $g$ admits a fixed point $\bar{x}$ satisfying $\bar{x}=g(\bar{x})$, representing a candidate limit of the algorithm. The key question is whether orbits starting sufficiently close to $\bar{x}$ converge to $\bar{x}$, and how this behavior depends on the update scheme.  To analyze local convergence, we study small perturbations around the fixed point. Writing $x_k=\bar{x}+y_k$ and linearizing the iteration yields
\begin{eqnarray}\label{eq:map}
y_{k+1} = Ay_k, \qquad A := Dg(\bar{x}),
\end{eqnarray}
where the Jacobian $A$ depends on the specific form of $g$. In particular, sequential and parallel updates typically lead to different Jacobians, even when they share the same fixed point. The orbits $\{A^k y_0\}_{k\ge 0}$ of this linear system approximate the behavior of the nonlinear iteration near $\bar{x}$ and determine whether perturbations decay or amplify.
The fixed point $\bar{x}$ is called \emph{Lyapunov stable} if all orbits starting sufficiently close remain close for all future iterations, and \emph{asymptotically stable} if, in addition, these orbits converge to $\bar{x}$ as $k\to\infty$.  This notion captures local convergence of the algorithm under the chosen update scheme. The guarantee of local convergence is quantified \citep{wiggins2003introduction} by the spectral radius of $A$, denoted by $\rho(A) = \max\{ | \lambda |: \lambda \, \text{is an eigen value of}\, A\}$. 
\begin{theorem}\label{thm:sr}
If $\rho\{Dg(\bar{x})\} < 1$, then the fixed point $\bar{x}$ of the nonlinear map \eqref{eq:maps} is asymptotically stable.
\end{theorem}

Theorem~\ref{thm:sr} provides a standard sufficient condition for local asymptotic stability of a fixed point in terms of the spectral radius of the Jacobian of the update map. In the context of variational algorithms, this result allows us to characterize the local convergence behavior of coordinate ascent schemes by analyzing the linearization of the corresponding update operators at a stationary point.

In the following, we compute the Jacobian matrices associated with the sequential and parallel CAVI updates and study their spectral properties at a fixed point $\mu^*$.  For simplicity, assume $\sigma^2 =1$. The following calculations form the basis for a precise comparison between the two schemes.

We begin with the sequential update \eqref{eq:seq_final}. The Jacobian of the nonlinear map at the fixed point $\mu^*$ is given by
\begin{align}\label{eq:seq_J}
J_{\mathrm{seq}}(\mu^*) = G(\mu^*) 
+ \Big[ \frac{\partial G(\mu^*)}{\partial \mu_1} \mu^{*}; \cdots; \frac{\partial G(\mu^{*})}{\partial \mu_1}  \mu^{*}\Big] 
+ \Big[ \frac{\partial H(\mu^{*})}{\partial \mu_1}; \cdots; \frac{\partial H(\mu^{*})}{\partial \mu_p}\Big], 
\end{align}
where $\alpha_j^* = \psi_j(\mu_j^*)$. The first term corresponds to the linear part of the iteration, while the remaining terms arise from the implicit dependence of the operators $G(\cdot)$ and $H(\cdot)$ on the variational inclusion probabilities.

To evaluate these derivatives, we make repeated use of the identity that for a matrix-valued function $A(x) \in \mathbb{R}^{n \times n}$,
$\partial A(x)^{-1} / \partial x = - A(x)^{-1} [\partial A(x)/ \partial x] A(x)^{-1}.
$
Applying this identity yields
\begin{align*}
\frac{\partial G(\mu^{*})}{\partial \mu_j}  
=   (D+ L^*D_{\alpha^{*}})^{-1} L^* \frac{\partial D_{\alpha^{*}}}{\partial \mu_j} (D+ L^*D_{\alpha^*})^{-1}(L^*)^{\T} D_{\alpha^{*}} 
- (D+ L^*D_{\alpha^{*}})^{-1} (L^*)^{\T}  \frac{\partial D_{\alpha^*}}{\partial \mu_j}, 
\end{align*}
and
\begin{align*}
\frac{\partial H(\mu^{*})}{\partial \mu_j} 
=   -(D+ L^*D_{\alpha^{*}})^{-1} L^* \frac{\partial D_{\alpha^{*}}}{\partial \mu_j} (D+ L^*D_{\alpha^{*}})^{-1} f.
\end{align*}
These expressions make explicit how the nonlinearity induced by the spike-and-slab variational parameters propagates into the Jacobian through the diagonal matrix $D_{\alpha^*}$.

We next turn to the parallel update \eqref{eq:par_final}. In this case, the Jacobian at the fixed point $\mu^*$ admits a simpler closed form:
\begin{align}\label{eq:par_J}
J_{\mathrm{par}}(\mu^{*}) 
= - D^{-1}(L + L^{\T}) D_{\alpha^*} 
- D^{-1}(L + L^{\T}) \mbox{diag} (\dot{\alpha}_j^{*} \mu_j^{*}),
\end{align}
where $\dot{\alpha}_j^{*}$ denotes the derivative of $\alpha_j$ evaluated at $\mu_j^*$. 

Our goal is to analyze the two spectra, given by  $\rho\{J_{\mathrm{seq}}(\mu^*)\}$ and $\rho\{J_{\mathrm{par}}(\mu^{*})\}$ and thereby contrast the local stability properties of the two algorithms. To facilitate this analysis, we introduce the following notation.  Let $L_1^* = D^{-1/2} L^* D^{-1/2}$ and define
$
A := L_1^*  + (L_1^*)^{\T}, 
$  $B = \mbox{diag}\{(\mu_j^*)^2 a_j(1- d_j^*)\}$, $
C = D_{\alpha^*}^{1/2}AB.$
Observe that the matrix
$
M := A + I
$
is positive definite by construction.
To characterize the local stability of the sequential CAVI updates, it is necessary to control the interaction between the curvature induced by the variational inclusion probabilities and the correlation structure of the design matrix. In particular, the nonlinear dependence of $D_{\alpha^*}$ on the fixed point $\mu^*$ introduces higher-order terms into the Jacobian that cannot be neglected a priori. 

The following Assumption \ref{ass:contraction} formalizes a regime in which these nonlinear effects remain sufficiently mild. The condition can be viewed as a localized contraction requirement that bounds quadratic forms involving the matrix $B$ when the signals are either sufficiently large or sufficiently small, which capturing the sensitivity of the variational inclusion probabilities relative to the positive definite matrix $M = A + I$.  We explain the plausibility of Assumption \ref{ass:contraction} in Remark \ref{rem:ass_pl}. 

\begin{assumption}\label{ass:contraction}
Let $K_{L_1^*}  =  y^{H} D_{\alpha^*} (L_1^*)^{\T} L_1 D_{\alpha^*} y$. 
For all  y with $y^Hy = 1$, there exists $0 < \delta <  \min \{1/2, \lambda_{\mathrm{min}} [M + D_{\alpha^*}^{-1}]/ 
\| (L_1^*)^{\T} L_1^* \|_2 \}$  such that 
\begin{eqnarray*}
y^{H} B M^2 B y \leq \delta y^{H} M  y, \quad   
y^{H} B^2y \leq \delta y^{H} (D_{\alpha^*}^{-1} - I)  y
\end{eqnarray*}
\end{assumption}
We are now in a position to state the main stability result for the sequential CAVI scheme.
\begin{theorem}\label{thm:seq_cont}
If Assumption \ref{ass:contraction} is satisfied, then $\rho\{J_{\mathrm{seq}}(\mu^*)\} < 1$.  
\end{theorem}
Theorem~\ref{thm:seq_cont} establishes that, under mild regularity conditions, the sequential CAVI algorithm is locally asymptotically stable at the fixed point $\mu^*$. Combined with Theorem~\ref{thm:sr}, this result implies local linear convergence of the iterates.
\begin{remark}[Plausibility of Assumption \ref{ass:contraction}]\label{rem:ass_pl}
Fix $\delta\in(0,1)$ and suppose that $X\in\mathbb R^{n\times p}$  has i.i.d.\ $N(0,1)$ entries with $p \leq n$, so that $a_j=\|X_j\|^2+\tau=O(n)$ uniformly with high probability and $\|M\|_2\le 1+C\sqrt{p/n}$, $\lambda_{\min}(M)\ge 1/2$.  Let $p \leq n$. Recall that $B=\mathrm{diag}(b_j)$ with $b_j=(\mu_j^*)^2 a_j(1-d_j^*)$. Also, observe that $(1-d_j^*)/d_j^*= \mbox{logit}(\pi)^{-1}\sqrt{a_j}e^{-a_j(\mu_j^*)^2/2}$ and $b_j = \mbox{logit}(\pi) a_j (\mu_j^*)^2/ (1+ \mbox{logit}(\pi) e^{a_j (\mu_j^*)^2/2}/\sqrt{a_j})$.  Then, with $\pi = 1/p^A$,  if $|\mu_j^*|\le \kappa_1 n^{-1/2}$, then $b_j\le C_1\kappa_1^2$ and  $(1-d_j^*)/d_j^* \geq \sqrt{a_j}e^{-\kappa_1^2/2}$.  If $|\mu_j^*| >  \kappa_2 n^{-1/2}\log^{1/2} n$ for sufficiently large $\kappa_2$ depending on $\delta$, $b_j^2 \leq \delta (1-d_j^*)/d_j^*$ for $\delta \in (0, 1)$.  Hence by choosing $\kappa_1$ sufficiently small, depending on $\delta$, we have 
\begin{align*}
y^H B M^2 B y &\le (1+C\sqrt{p/n})^2 C_1^2 \kappa_1^4 \|y\|^2 \le \delta\, y^H M y, \\
y^H B^2 y & \le y^H (D_{\alpha^*}^{-1}-I)y. 
\end{align*}

Hence Assumption \ref{ass:contraction} holds with a high probability probability in both small and large signal regimes. The intermediate regime 
$
|\mu_j^*| \in (\kappa_1 n^{-1/2}, \kappa_2 n^{-1/2}\log^{1/2} n)
$
is more delicate, as neither the small nor large-signal approximations apply directly. We do not pursue a detailed analysis of this regime here.
\end{remark}

In contrast, as shown in Theorem \ref{thm:par_cont}, the parallel CAVI scheme may fail to satisfy a contraction property even at a stationary point. Indeed, the spectral radius of the Jacobian at the fixed point exceeds one, which explains the divergence of the iteration observed in \S \ref{ssec:ex}.
\begin{theorem}
\label{thm:par_cont}
Let $X = (X_1,\dots,X_p) \in \mathbb{R}^{n \times p}$ have i.i.d.\ $\mathcal N(0,1)$ entries.  Then if all the elements of $D_{\alpha^*}$ are  greater than $1- \varepsilon$ for some $\varepsilon \in (0, 1)$,  there exist constants $c,c_1,c_2>0$ such that, for all sufficiently large
$n,p$,
\[
\mathbb P\!\left( \rho\{J_{\mathrm{par}}(\mu^*)\} > 2(1- \varepsilon)\sqrt{p/n} \right)
\;\ge\;
1 - 2e^{-c_1 n} - 2e^{-c_2 p}.
\]
\end{theorem}

\begin{remark}\label{rem:neg}
If all the true coefficients are large, then there exists $\varepsilon \in (0, 1)$ such that all the elements of $D_{\alpha^*}$ are greater than $1- \varepsilon$. So the assumptions of Theorem \ref{thm:par_cont}  are satisfied. Then if $p > n/\{4 (1- \varepsilon)^2\}$, $\rho\{J_{\mathrm{par}}(\mu^*)\}  > 1$ with high probability. 
\end{remark}

\section{Numerical study of local stability}\label{ssec:ex}
We now empirically investigate the local stability properties predicted by the Jacobian analysis in Section 4. Recall that convergence of the fixed-point iteration $\mu^{(t+1)} = g(\mu^{(t)})$ is governed by the spectral radius of the Jacobian $J = Dg(\mu^\ast)$ at a fixed point $\mu^\ast$; in particular, $\rho(J) < 1$ implies local contraction.
To assess this behavior, we compute the spectral radii of the Jacobians $J_{\mathrm{seq}}(\mu^\ast)$ and $J_{\mathrm{par}}(\mu^\ast)$ given in \eqref{eq:seq_J} and \eqref{eq:seq_J} across multiple simulated datasets. Figure~\ref{fig:spectral_radius} summarizes the distribution of $\log \rho(J)$ over $50$ replications.
\begin{figure}[t]
\centering
\includegraphics[width=0.9\textwidth]{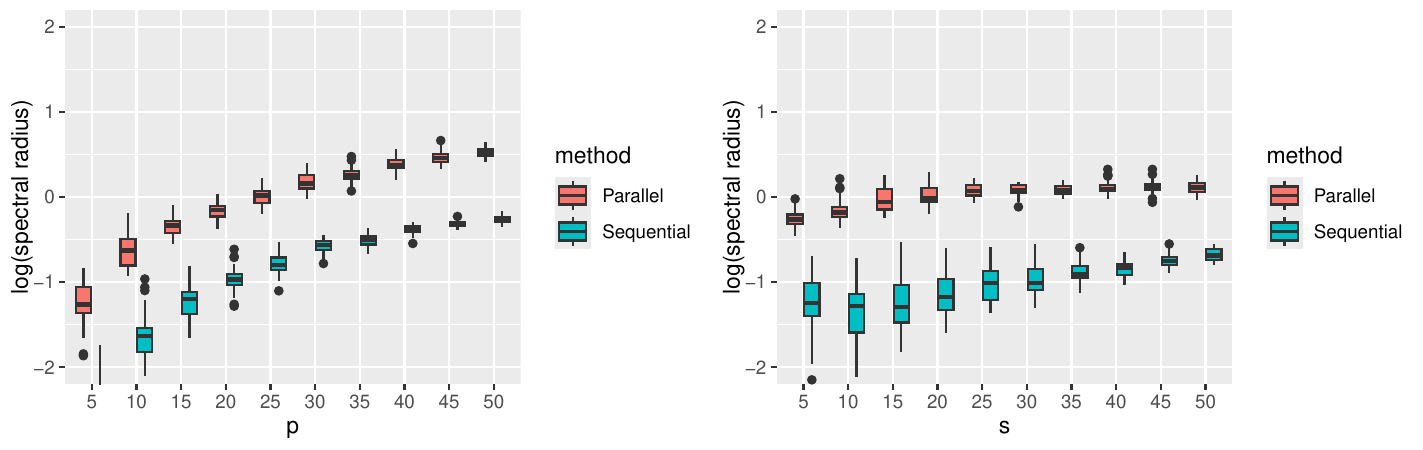}
\caption{Distribution of $\log \rho(J)$ for sequential and parallel CAVI. Left: varying $p$ with $s=p$ and fixed $n =100$. Right: varying $s$ with fixed $(n=200, p=50)$.}
\label{fig:spectral_radius}
\end{figure}
In the left panel, we vary the dimension $p \le n$ with $s=p$ and fixed $n$, while in the right panel we vary the sparsity level $s$ for fixed $(n,p)$. In both settings, a sharp contrast emerges between the two update schemes.
For sequential CAVI, the spectral radius remains uniformly below one, i.e., $\log \rho(J_{\mathrm{seq}}(\mu^\ast)) < 0$, indicating that the fixed point is locally contractive. This is consistent with the stable and monotone ELBO behavior observed in the running example.
In contrast, the parallel scheme frequently yields $\rho(J_{\mathrm{par}}(\mu^\ast)) > 1$, with $\log \rho$ increasing as either $p$ or $s$ grows.   In particular, when $s= p$, and $p > n/4$, we observe $\rho(J_{\mathrm{par}}(\mu^\ast)) > 1$ with a high probability, provide strong empirical support for the theoretical findings in  Remark \ref{rem:neg}. 

In light of Section~4.2 of \cite{ray2022variational}, where the sequential CAVI scheme may fail under random initialization, one might perceive a contradiction with our results. This is not the case. Our analysis is local.  For the sequential updates, we show $\rho(J_{\mathrm{seq}}(\mu^*))<1$, implying local stability only for initializations in a neighborhood of $\mu^*$. In contrast, for the parallel updates, $\rho(J_{\mathrm{par}}(\mu^*))>1$ under certain regimes, yielding divergence irrespective of initialization. Thus, while our negative result for the parallel scheme is universal, the positive result for the sequential scheme is local and does not preclude failure under poor initialization. 

\section{Discussion}
The theoretical guarantees established in this paper rely on assumptions that are most naturally satisfied under regimes where $p \le n$, ensuring that the design matrix exhibits favorable spectral properties. Additionally, our results do not explicitly characterize how sparsity in the underlying signal influences the stability properties of the algorithms.
As highlighted in Remark \ref{rem:neg}, the negative result for the parallel scheme is most clearly justified under a dense regime in which all signals are active, leading to variational inclusion probabilities close to one and, consequently, instability of the associated fixed point. This regime represents a worst-case scenario and may not fully reflect practical settings where sparsity is present. Empirically, we observe that when the number of true signals is small relative to $p$, the parallel algorithm can exhibit improved stability behavior for moderate values of $p$. This suggests that sparsity may play a regularizing role in the dynamics of the algorithm, potentially mitigating the divergence phenomena predicted by the theory in dense settings.
A precise theoretical understanding of how sparsity interacts with the local stability properties of sequential and parallel CAVI remains an important open problem. 

\section{Acknowledgements}
I thank Anirban Bhattacharya, Yun Yang, Natesh Pillai and Botond Szabo for helpful discussions on this topic. I also note that generative AI tools were used for producing certain plots and visualizations; however, all theoretical results and proofs are entirely the author's own.

\appendix

\section{Proof of Theorem \ref{thm:seq_cont}}
Observe that the equation  $J_{\mathrm{seq}}(\mu^*)x = \lambda x$ can be written as
  \begin{align*}
 \lambda (D+ L^*D_{\alpha^{*}})x = -(L^*)^{\T}  D_{\alpha^{*}} x  - L^* \mbox{diag}(x_j \dot{\alpha}_j^{*}) \mu^{*} - 
 (L^*)^{\T}  \mbox{diag}(x_j \dot{\alpha}_j^{*}) \mu^{*}. 
 \end{align*}
Recall that  $L_1^* = D^{-1/2} L^* D^{-1/2}$. Hence 
   \begin{align*}
 \lambda D^{1/2} (I+ L_1^*D_{\alpha^{*}}) D^{1/2}x = -(L^*)^{\T}  D_{\alpha^{*}} x  - L^* \mbox{diag}(x_j \dot{\alpha}_j^{*}) \mu^{*}. 
 (L^*)^{\T}  \mbox{diag}(x_j \dot{\alpha}_j^{*}) \mu^{*}  
 \end{align*}
 Setting $y = D^{1/2}x$, we have 
   \begin{align*}
 \lambda  (I+ L_1^*D_{\alpha^{*}})y &= -(L_1^*)^{\T}  D_{\alpha^{*}} y  - [L_1^* + (L_1^*)^{\T}] \mbox{diag}(y_j \dot{\alpha}_j^{*}) \mu^{*} \\
 &= -(L_1^*)^{\T}  D_{\alpha^{*}} y  - [L_1^* + (L_1^*)^{\T}] \mbox{diag}(\mu_j^*\dot{\alpha}_j^{*}) y. 
 \end{align*}
Letting further $D_\alpha^*y =  z$ and denoting $D_{\alpha^*} = \mbox{diag}(d_j^*)$
   \begin{align*}
 \lambda  z^{\H} (D_{\alpha^*}^{-1}+ L_1^*)z = -z^{\H}(L_1^*)^{\T}  z  - z^{\H}[L_1^* + (L_1^*)^{\T}] \mbox{diag}(z_j \dot{\alpha}_j^{*}/d_j^{*}) \mu^{*}. 
 \end{align*}
Denote the real number $y^{\H} D_{\alpha^*}y = z^{\H} D_{\alpha^*}^{-1}z$ by $s^* \in (0, 1)$, $z^{\H} L_1^* z$ by $a+\mi b$ and $c =
 y^{\H} D_{\alpha^*}[L_1^* + (L_1^*)^{\T}] \mbox{diag}(\mu_j^* \dot{\alpha}_j^{*}) y$. Then 
\begin{align*}
\lambda = \frac{-a+ \mi b + c}{s^*+ a+ \mi b},
\end{align*}
so that 
\begin{align*}
|\lambda|^2 \leq  \frac{a^2 + b^2 + |c|^2 + 2|c| \sqrt{a^2 + b^2}}{a^2 + b^2 + 2as^* +(s^*)^2},
\end{align*}
implying $|\lambda| < 1$ if $|c|^2 + 2|c|\sqrt{a^2 + b^2} < (2a+s^*)s^*$.  
Recall that $A:= L_1^*  + (L_1^*)^{\T}$, $B = \mbox{diag}\{(\mu_j^*)^2 a_j(1- d_j^*)\}$ and $C = D_{\alpha^*}^{1/2}AB$. Observe that $M := A + I$ is a positive definite matrix.  Then, $c =
 y^{\H} D_{\alpha^*} AB D_{\alpha^*} y$. 
Note that by Cauchy-Schwarz inequality,   
\begin{eqnarray*}
|c|^2  &=& |y^{H} D_{\alpha^*}^{1/2} D_{\alpha^*}^{1/2} AB D_{\alpha^*}^{1/2} D_{\alpha^*}^{1/2} y|^2 \\
&\leq& [y^{H} D_{\alpha^*} y ] \times [y^{H} D_{\alpha^*} B^\T A^\T D_\alpha AB D_{\alpha^*}  y] \\
&\leq& s^* y^{H} D_{\alpha^*} C^{\T} C D_{\alpha^*}   y \leq s^* y^{H} D_{\alpha^*} C^{\T} C D_{\alpha^*}   y  
\end{eqnarray*}
Also, 
\begin{eqnarray*}
a^2 + b^2 = |z^{\H} L_1^* z|^2 &\leq& \|z\|^2 \|L_1^* z\|^2   \\
&=&  [y^{H} D_{\alpha^*}^2 y] \times  [y^{H} D_{\alpha^*} (L_1^*)^{\T} L_1^* D_{\alpha^*} y]  \\
&\leq &  s^* K_{L_1^*}. 
\end{eqnarray*}
From  Lemma \ref{lem:aux3},  
$|c|^2 \leq \delta s^* (2a + s^*)$ and $2|c|\sqrt{a^2 + b^2} \leq 2 \sqrt{\delta} s^*  \sqrt{2a + s^*}   \sqrt{K_{L_1^*}}$. 
Observe that $(2 a + s^*)  / K_{L_1^*}$ is 
\begin{eqnarray*}
\frac{y^{H} D_{\alpha^*} (A + D_{\alpha^*}^{-1})D_{\alpha^*} y}{ y^{H} D_{\alpha^*} (L_1^*)^{\T} L_1^* D_{\alpha^*} y} \geq \frac{\lambda_{\mathrm{min}} [A + D_{\alpha^*}^{-1}]}{\| (L_1^*)^{\T} L_1^* \|_2}. 
\end{eqnarray*}  
Hence
\begin{eqnarray*}
|c|^2 + 2|c|\sqrt{a^2 + b^2} < (2a+s^*)s^*  \Bigg( \delta + 2\sqrt{\frac{\delta K_{L_1^*}}{2a +s^*}}\bigg) 
< (2a+s^*)s^*
\end{eqnarray*}
where the last inequality follows from Assumption 1, Lemmas \ref{lem:aux1} and \ref{lem:aux3}. 

\section{Proof of Theorem \ref{thm:par_cont}}
Recall that 
\begin{align}\label{eq:Jacobian_par}
J_{\mathrm{par}}(\mu^{*}) = - D^{-1}(L + L^{\T}) D_{\alpha^*} - D^{-1}(L + L^{\T}) \mbox{diag} (\dot{\alpha}_j^{*} \mu_j^{*}).
\end{align}
Then, 
\begin{align}
D^{1/2}J_{\mathrm{par}}(\mu^{*})D^{-1/2} = -(L_1^* + (L_1^*)^{\T}) (I + B) D_{\alpha^*} = (I - M)(I + B)D_{\alpha^*}.
\end{align}

Let $\widetilde{D} := (I+B)D_{\alpha^*}$. Since 
\begin{align}
D^{1/2}J_{\mathrm{par}}(\mu^{*})D^{-1/2} = -(L_1^* + (L_1^*)^{\T}) (I + B) D_{\alpha^*} = (I - M)\widetilde{D}
\end{align}
and $(I - M)\widetilde{D}$ has the same eigen values as $\widetilde{D}^{1/2}(I - M)\widetilde{D}^{1/2}$.  Let $\tilde{d}_{\mathrm{min}}$ be  the minimum of the diagonal elements of $\widetilde{D}$.  Then from the assumption, $\tilde{d}_{\mathrm{min}} > 1 - \varepsilon$. 
 By Lemma \ref{thm:Wigner_lower_tau}, $\|I-M\|_2 \ge 2 \sqrt{p/n}$ with probability at least $1-e^{-Cp} - e^{-cn}$. It follows that on the high probability event, $J_{\mathrm{par}}(\mu^*) \geq  \tilde{d}_{\mathrm{min}}  2 (1- \varepsilon)\sqrt{p/n}$.

\section{Auxiliary results}

\begin{proposition}\label{prop:mixture}
Let $a > 0$ and $b \in \mathbb{R}$. Consider the density 
\be 
p(x) \propto e^{-\frac{1}{2}(ax^2 - 2bx)} \, \big[(1-\alpha) \delta_0(x) + \alpha N(x; 0, \tau^{-1}) \big]. 
\ee 
Then, 
\be 
p \equiv (1 - \widetilde{\alpha}) \delta_0 + \widetilde{\alpha} N\bigg(\frac{b}{a+\tau}, \frac{1}{a+\tau}\bigg), 
\ee
where 
\be 
\frac{\widetilde{\alpha}}{1 - \widetilde{\alpha}} = \frac{\alpha}{1-\alpha} \, \bigg(\frac{\tau}{a+\tau}\bigg)^{1/2} \, \exp\bigg(\frac{b^2}{2(a+\tau)}\bigg). 
\ee
\end{proposition}

\begin{proposition}\label{prop:cavi}
A generic form of the CAVI update for the $q_j$ is given by 
\begin{align*}
q_j(\beta_j) 
 =  (1 - \alpha_j) \delta_0 + \alpha_j N\big(\beta_j; \mu_j, a_j^{-1})
 \end{align*}
 where $
a_j = \|X_j\|^2/\sigma^2 + \tau, \ \mu_j = b_j/\{(\sigma^2)a_j\} , \mbox{logit}(\alpha_j) = \mbox{logit}(\pi) + (1/2) \log (\tau /a_j)+ b_j^2/\{\sigma^4 a_j^2\}$ and 
$b_j = \langle y - X_{-j} \mb E_{q_{-j}} \beta_{-j}, X_j \rangle$. 
\end{proposition}

\begin{proof}

Fixing $j$, let us compute $\int q_{-j} \log \pi_n$. To that end, define $y_{-j} = y - X_{-j} \beta_{-j}$. Then, 
$
\|y - X \beta\|^2 = \|y_{-j} - X_j \beta_j \|^2 = \|X_j\|^2 \beta_j^2 - 2 \langle y_{-j}, X_j \rangle \beta_j + \|y_{-j}\|^2 
$
Thus, we can write 
$
\log \pi_n(\beta) = -\frac{1}{2} \big[\|X_j\|^2 \beta_j^2 - 2 \langle y_{-j}, X_j \rangle \beta_j\big] + \log p(\beta_j) + \text { terms free of } \beta_j,
$
and hence
$
\int q_{-j} \log \pi_n = -\frac{1}{2} \big[\|X_j\|^2 \beta_j^2 - 2 \langle \mb E_{q_{-j}} y_{-j}, X_j \rangle \beta_j\big] + \log p(\beta_j) + C.
$
Thus, 
$
q_j(\beta_j) \propto \exp \bigg( \int q_{-j} \log \pi_n \bigg) \propto e^{-\frac{1}{2}(a \beta_j^2 - 2b \beta_j)} \, \big[(1-\pi) \delta_0(\beta_j) + \pi N(\beta_j; 0, \tau^{-1}) \big],
$
where $a_j = \|X_j\|^2$ and $b_j = \langle y - X_{-j} \mb E_{q_{-j}} \beta_{-j}, X_j \rangle$. 
From Proposition \ref{prop:mixture} we know that $q_j$ is a mixture of a point mass at zero and a Gaussian. Hence, we obtain 
\be 
q_j^{(t+1)}(\beta_j) \propto \exp \bigg( \int q_{-j}^{(t)} \log \pi_n \bigg) \propto e^{-\frac{1}{2}(a_{0j} \beta_j^2 - 2b_{0j}^{(t)} \beta_j)} \, \big[(1-\pi) \delta_0(\beta_j) + \pi N(\beta_j; 0, \tau^{-1}) \big],
\ee
where $a_{0j} = \|X_j\|^2$ and $b_{0j}^{(t)} = \langle y - X_{-j} \mb E_{q_{-j}^{(t)}} \beta_{-j}, X_j \rangle$. Thus, we have, $
q^{(t+1)} = \prod_{j=1}^p q_j^{(t+1)},  q_j^{(t+1)}(\beta_j) = (1 - \alpha_j^{(t+1)}) \delta_0 + \alpha_j^{(t+1)} N\big(\beta_j; \mu_j^{(t+1)}, a_j^{-1} \big)$, 
where, using Proposition \ref{prop:mixture}
\be 
a_j = \|X_j\|^2 + \tau, \ \mu_j^{(t+1)} = \frac{b_{0j}^{(t)}}{a_j}, \ \mbox{logit}(\alpha_j^{(t+1)}) = \mbox{logit}(\pi) + \frac{1}{2} \log \bigg(\frac{\tau}{a_j}\bigg) + \frac{(b_{0j}^{(t)})^2}{2 a_j}.
\ee
Since $\mb E_{q_k^{(t)}} \beta_k = \alpha_k^{(t)} \mu_k^{(t)}$, we can also express the dynamics above as  
\be 
\mu_j^{(t+1)} & = \frac{1}{a_j} \big[\langle X_j, y \rangle  - \langle X_{-j} (\alpha_{-j}^{(t)} \circ \mu_{-j}^{(t)} ), X_j \rangle \big], \\
\mbox{logit}(\alpha_j^{(t+1)}) & = \mbox{logit}(\pi) + \frac{1}{2} \log \bigg(\frac{\tau}{a_j}\bigg) + \frac{a_j (\mu_j^{(t+1)})^2}{2}. 
\ee
\end{proof}

\begin{lemma}\label{lem:aux1}
If $a = \mbox{Re}(z^{\H} L_1^* z)$ and $s^* =y^{\H} D_{\alpha^*}y$ , then $2a+s^* > 0$. 
\end{lemma}
\begin{proof}
Observe that 
\begin{align*}
2a+ s^* &= z^{\H} (L_1^*)^{\T}z+ z^{\H} L_1^*z+ z^{\H} D_{\alpha^*}^{-1}z \\
&= y^{\H} D_{\alpha^*}(L_1^*)^{\T}D_{\alpha^*} y+ y^{\H} D_{\alpha^*} L_1^* D_{\alpha^*}y+ y^{\H} D_{\alpha^*}y \\
& = y^{\H} D_{\alpha^*} M  D_{\alpha^*}y+ y^{\H} D_{\alpha^*}  (I - D_{\alpha^*}) y > 0
\end{align*}
since $M$ is a positive definite matrix. 
\end{proof}

\begin{lemma}\label{lem:aux3}
Under Assumption \ref{ass:contraction},  for any $y$ with $y^{H}y =1$, $y^H D_{\alpha^*} C^{\T} C D_{\alpha^*}y < \delta (2 a + s^*)$. 
\end{lemma}

\begin{proof}
First, observe that it is enough to consider $y$ to be a real vector.  In that case, 
\begin{eqnarray*}
y^{H} B (M - I)D_{\alpha^*} (M - I)B y = y^{\T}[ B\{ M D_{\alpha^*} M - 2 M  D_{\alpha^*} +  D_{\alpha^*} \}B] y.
\end{eqnarray*}
Since, 
$ y^{\T} B M D_{\alpha^*} M B y \leq  y^{\T}  B M^2 B y$ and $y^{\T}[ B(M  D_{\alpha^*} +  D_{\alpha^*} M) B] y > 0$ and from Assumption 1, $y^{\T}[  B D_{\alpha^*}B] y \leq y^{\T}  \leq \delta y^{\T}[  D_{\alpha^*}^{-1}  - I] y \leq y^{\T}  $. Hence from Assumption 1 again, 
\begin{eqnarray*}
y^{\T}B (M - I) D_{\alpha^*} (M - I) B y \leq  \delta y^{H} [ D_{\alpha^*}^{-1}  + (M - I)] y.
\end{eqnarray*}

\end{proof}

\begin{lemma}
\label{thm:Wigner_lower_tau}
Let $X \in \mathbb{R}^{n \times p}$ have i.i.d.\ $N(0,1)$ entries and let $\tau>0$. Define
\[
S := X^\top X, \qquad D_\tau := \mathrm{diag}(S) + \tau I, \qquad A_\tau := D_\tau^{-1/2}(S-D_\tau)D_\tau^{-1/2}.
\]
Assume $p \ge 4$ and $n \ge C_0 \log p$ for a sufficiently large universal constant $C_0$. Then there exist universal constants $c, C >0$ such that
\[
\mathbb{P}\Big( \|A_\tau\|_2 \ge c \sqrt{p/n} \Big) \ge 1 - e^{-c n} - e^{-C p}.
\]
\end{lemma}

\begin{proof}
For each $j$, $D_{\tau,jj} = \|X_j\|^2 + \tau \sim \chi_n^2 + \tau$. By the Laurent--Massart inequality \citep{LaurentMassart2000}, for any $t>0$,
$
\mathbb{P}\big(|\|X_j\|^2 - n| \ge 2\sqrt{nt} + 2t \big) \le 2e^{-t}.
$
Taking $t = \log p$ and applying a union bound over $j=1,\dots,p$, we obtain that with probability at least $1 - 2p^{-1}$,
$
\forall j:\quad \frac{n}{2} \le \|X_j\|^2 \le \frac{3n}{2}.
$
Hence, on this event,
$
c_1 n \le D_{\tau,jj} \le C_1 n,
$
and therefore
$
D_\tau^{-1/2} = n^{-1/2} I + E,
$
where $\|E\|_2 \le C \sqrt{\frac{\log p}{n}}$.
Write
$
A_\tau = D_\tau^{-1/2}(S - \mathrm{diag}(S))D_\tau^{-1/2}.
$
Using the approximation from Step 1,
$
A_\tau = \frac{1}{n}(S - \mathrm{diag}(S)) + R,
$
where $\|R\|_2 \le C \frac{\sqrt{\log p}}{n}$ with high probability.
It is well known (see, e.g., \citealp{BaiYin1988,Vershynin2018}) that for a Gaussian matrix $X$,
$
\left\| \frac{1}{n} X^\top X - I_p \right\|_2 \ge c \sqrt{\frac{p}{n}}
$
with probability at least $1 - e^{-c n} - e^{-C p}$.
Since
$
\frac{1}{n}(S - \mathrm{diag}(S)) 
= \left(\frac{1}{n} S - I_p\right) - \left(\frac{1}{n}\mathrm{diag}(S) - I_p\right),
$
we have
\[
\left\| \frac{1}{n}(S - \mathrm{diag}(S)) \right\|_2
\ge \left\| \frac{1}{n} S - I_p \right\|_2 
- \left\| \frac{1}{n}\mathrm{diag}(S) - I_p \right\|_2.
\]
By concentration of $\chi^2$ variables,
\[
\left\| \frac{1}{n}\mathrm{diag}(S) - I_p \right\|_2 
\le C \sqrt{\frac{\log p}{n}}
\]
with high probability. Combining the bounds yields
\[
\left\| \frac{1}{n}(S - \mathrm{diag}(S)) \right\|_2
\ge c \sqrt{\frac{p}{n}}.
\]
Finally, since $\|R\|_2$ is of smaller order, we conclude that
\[
\|A_\tau\|_2 \ge c \sqrt{\frac{p}{n}}
\]
with probability at least $1 - e^{-c n} - e^{-C p}$.
\end{proof}

\bibliographystyle{biometrika}
\bibliography{paper-ref,CAVI}


%
%
%
\end{document}